\documentclass[a4paper,11pt]{article}
\usepackage[margin=1in]{geometry}

\usepackage[utf8]{inputenc} 
\usepackage[T1]{fontenc}    
\usepackage{hyperref}
\usepackage{url}            
\usepackage{booktabs}       
\usepackage{amsfonts,amsmath,amsthm}       
\usepackage{nicefrac}       
\usepackage{microtype}      
\usepackage{xcolor}         
\usepackage{csquotes}
\usepackage{authblk}
\usepackage{natbib}
\usepackage{tcolorbox} 
\usepackage{graphicx}
\usepackage{wrapfig}
\usepackage{relsize}
\usepackage{color}
\usepackage{pict2e}
\usepackage[font=small]{caption}
\usepackage{subcaption}
\usepackage{algorithm}
\usepackage[noend]{algorithmic}
\usepackage{nameref}
\usepackage{makecell}

\usepackage{enumitem}

\newcommand{\R}{\mathbb{R}}

\newtheorem{theorem}{Theorem}

\newtheorem{proposition}[theorem]{Proposition}

\newtheorem{definition}{Definition}

\newcommand{\tr}{\mathrm{tr}}

\newcommand{\vocab}{\mathcal{V}}

\newcommand{\EffDim}{d_\rho}
\newcommand{\Hard}{\mathrm{Hard}}

\title{\textbf{Geometric Signatures of Reasoning: A Spectral Perspective on Task Hardness}}
\author[*]{Aria Masoomi}
\author[1,*]{Mahsa Bazzaz}
\author[2,3]{Adel Javanmard}
\author[3]{Vahab Mirrokni}
\affil[1]{Northeastern University}
\affil[2]{University of Southern California}
\affil[3]{Google Research}
\date{}

\begin{document}

\def\thefootnote{$*$}\footnotetext{Equal contribution}
\maketitle

\begin{abstract}
Chain-of-thought (CoT) reasoning enables large language models (LLMs) to solve
complex problems by generating intermediate reasoning steps. While much attention
has been paid to the length and content of these reasoning chains, far less is known about
their internal geometry. We study the \emph{geometry} of CoT trajectories in the
hidden state space of transformer models, formalizing each reasoning chain as a
discrete curve in $\R^d$ and characterizing it through spectral, positional, and
kinematic geometric functionals.
We introduce the effective dimension $\EffDim$ as a measure of trajectory
complexity and empirically show that harder tasks exhibit greater effective dimension. Lastly, we explore how kinematic features of the trajectory, mean position, positional dispersion, initial and final hidden states, mean velocity, mean speed, and speed dispersion, can be used to predict
solution correctness before generation is complete, and may inform future early-stopping strategies.
Experimentally, on 86 mathematical reasoning problems from the MATH500 dataset
and three Qwen2.5 model scales (0.5B, 1.5B, and 7B), $\EffDim$ achieves
difficulty-prediction AUROC of $0.90$, $0.94$, and $0.96$, respectively. Seven
kinematic and positional features predict correctness from only the first $20\%$
of generated tokens, with AUROC between $0.77$ and $0.88$ across scales. These
correctness signatures transfer across held-out questions, establishing that the
\emph{shape} of a model's internal reasoning trajectory is a principled window
into both task hardness and solution quality.
\end{abstract}

\section{Introduction}
\label{sec:intro}

Large language models (LLMs) have demonstrated remarkable reasoning capabilities
through chain-of-thought (CoT) prompting, where models generate intermediate
reasoning steps before producing final answers~\citep{wei2022chain}. Recent systems
such as OpenAI's o1 and DeepSeek R1 have shown that scaling test-time
compute, allowing models to ``think longer'', can dramatically improve performance
on complex reasoning tasks~\citep{openai2024learning, guo2025deepseek}. At the same time,
it has been observed that simply increasing test-time computation can harm
performance, a phenomenon known as overthinking: reasoning length does not
directly convert to correct answers~\citep{su2025between}. In general, one
expects a model to engage in more deliberate reasoning for harder tasks and less for
easier ones. Recent theoretical work has further shown that, for transformers trained
on an in-context weight prediction task for linear regression, increasing test-time
compute can harm performance when the skills required to solve the downstream task
are insufficiently represented in the training data~\citep{javanmard2025understanding}.

Despite these theoretical and empirical advances, fundamental questions remain:
What makes a task hard for a general LLM? How does task difficulty affect the
model's internal representations, and can we predict it from those representations
alone? Can we identify promising reasoning paths early, before generation is
complete? These questions have profound practical implications: when generating $n$
candidate solutions to a problem (best-of-$n$ sampling), can we prioritize which
paths to pursue based on the geometry of their early trajectories?

In this paper, we address these questions by studying the geometry of
chain-of-thought trajectories in the hidden state space of transformer models. Our
key insight is that as an LLM generates tokens during reasoning, the sequence of
hidden states traces a discrete curve in $\R^d$, and the \emph{geometric
properties} of this curve encode information about both task difficulty and solution
quality.

\begin{figure}[t]
    \centering
    \includegraphics[width=0.96\textwidth]{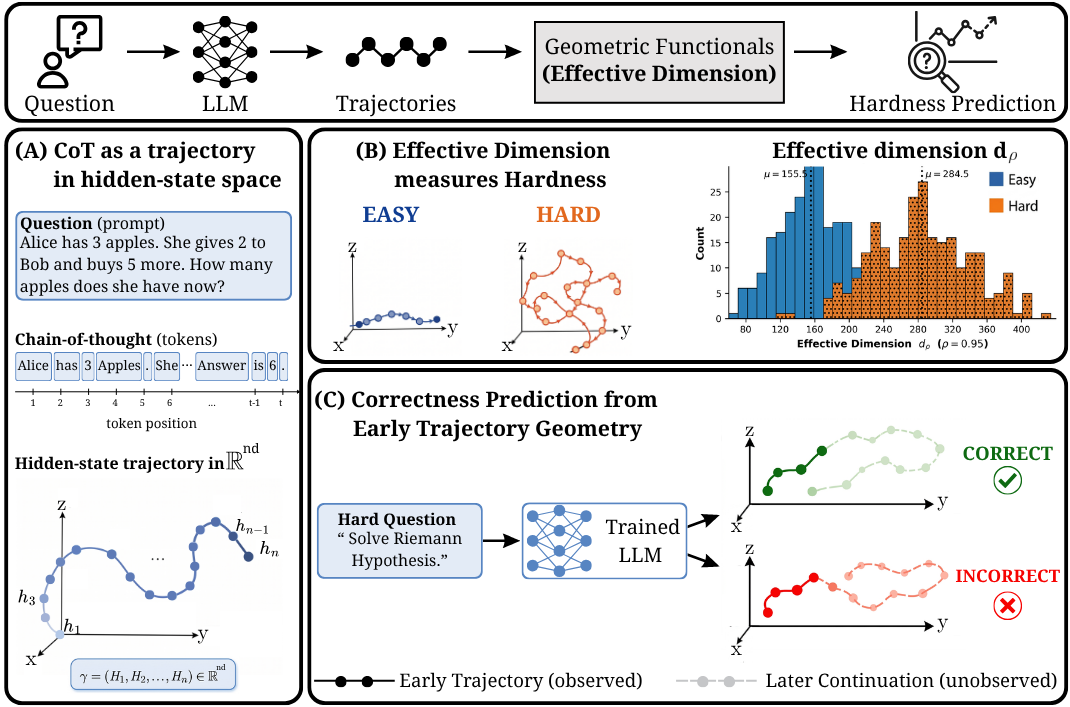}
    \caption{Overview of our framework. A question is fed to an LLM whose
    chain-of-thought generation traces a discrete curve $\gamma = (h_1, \ldots,
    h_n) \in \R^d$ in hidden-state space. \textbf{(A)} Each generated token
    produces a hidden state, forming a trajectory whose geometry we analyze.
    \textbf{(B)} Empirically, hard problems exhibit higher-dimensional trajectories than easy
    ones at every tested model scale, providing a geometric measure of task
    hardness. \textbf{(C)} Whether a reasoning chain will reach a correct
    answer is partially detectable from geometric functionals of the early
    trajectory: kinematic features extracted from only the first $20\%$ of
    generated tokens predict solution correctness with high AUC before
    generation is complete.}
    \label{fig:overview}
\end{figure}

Our contributions are as follows:
\begin{itemize}[leftmargin=0.2cm]
    \item Formal Framework for CoT Geometry (Section~\ref{sec:formulation}):
    We formalize CoT reasoning as a discrete curve in $\R^d$ and introduce geometric functionals 
    that extract spectral, positional, and kinematic properties of reasoning trajectories.
   
    \item Effective Dimension as Task Complexity (Section~\ref{sec:effdim}): We
    introduce a geometrical functional capturing hardness of the task. More precisely, the effective dimension $\EffDim$ of reasoning curves as a principled
    measure of task hardness. We further characterize which curves attain the highest
    effective dimension, establishing them as geometric representatives of the
    hardest tasks.
    \item Hardness Prediction (Section~\ref{sec:exp-difficulty}): Using only effective dimension features, we achieve AUROC
    $0.90$--$0.96$ in predicting whether a mathematical problem is easy or hard across three model scales.   
    \item Correctness Prediction
    (Section~\ref{sec:exp-correctness}): Seven 
    kinematic and positional features of the trajectory predict solution correctness 
    with AUROC $0.77$--$0.88$ from only the first $20\%$ of generated tokens, with 
    promising implications for early-exit strategies and best-of-$n$ ranking.
\end{itemize}
\section{Related Work}
\label{sec:related}

Chain-of-thought prompting~\citep{wei2022chain, kojima2022large} has emerged as a
powerful technique for eliciting multi-step reasoning in LLMs. Recent work has
explored scaling test-time compute~\citep{snell2024scaling, welleck2024decoding,
muennighoff2025s1}, with systems like OpenAI o1~\citep{openai2024learning} and DeepSeek
R1~\citep{guo2025deepseek} demonstrating strong performance through extended
reasoning chains. A complementary line of work has observed that more reasoning
is not always better: overthinking can degrade performance when the skills required
for a task are underrepresented in training~\citep{su2025between}. Our work
studies these phenomena from a geometric angle, asking not how long a chain is
but what shape it traces in hidden state space.

\citet{javanmard2025understanding} provide a theoretical analysis of test-time
scaling for transformers trained on in-context weight prediction for linear
regression. They characterize task hardness via the ratio of the trace with the
minimum eigenvalue of the feature covariance matrix, showing that harder tasks
require longer chains-of-thought to reach a given error level, and that
insufficient task coverage in training can cause additional reasoning steps to
hurt performance. Our work is complementary but distinct in two ways. First, we
study task hardness empirically in a general LLM rather than deriving it from a
tractable linear model. Second, and more fundamentally, we shift the unit of
analysis from the output chain to the internal hidden state trajectory: we show
that task hardness leaves a geometric signature in the model's representation
space, captured by the effective dimension $\EffDim$ of the trajectory covariance,
and that this quantity alone is highly predictive of problem difficulty. 

\citet{korbak2025chain} argue that chain-of-thought reasoning offers a
unique safety opportunity because, for sufficiently hard tasks, transformers must
externalize reasoning through the CoT in order to complete it, making that
reasoning in principle observable. They focus on the content of the generated text
as the monitoring signal and discuss conditions under which this signal may degrade.
Our work operates at a different level: rather than reading the textual content of
the chain, we read the \emph{geometry} of the hidden states that produce it. The
two perspectives are complementary, CoT text monitoring and hidden-state
trajectory analysis can in principle be combined, but our approach is
model-internal and does not rely on the model producing legible natural language
reasoning.

\citet{sun2026llm} study LLM reasoning as a structured trajectory in
representation space, extracting hidden states at explicit step boundaries
(``Step 1:'', ``Step 2:'', \ldots) and showing that these activations form
linearly separable, step-specific subspaces that become more pronounced with
layer depth. For correctness prediction, they achieve high AUC using
late-step trajectory features, and explore inference-time interventions such as
activation steering~\cite{turner2023steering} to correct deviating trajectories.
Our work shares the trajectory perspective but pursues different goals. Rather
than analyzing step-boundary activations, we treat the full token-level hidden
state sequence as a continuous curve and characterize it through spectral and
kinematic geometric functionals. This allows us to ask whether trajectory geometry encodes
task difficulty. We show that the effective dimension $\EffDim$ of the
trajectory covariance, a spectral property of the curve as a whole, predicts
whether a problem is easy or hard with high AUC and empirically find that
harder tasks exhibit higher-dimensional trajectories.
We further show that kinematic features of the trajectory carry an early
correctness signal that is detectable from only the first $20$ percent of generated
tokens, opening a practical route to early stopping and best-of-$n$ ranking
without waiting for generation to complete.

The closest correctness baselines are Chain-of-Embedding (CoE), which aggregates
the magnitude and angle of hidden-state changes across the \emph{layer} axis for a
response~\citep{wang2025coe}, and LLM-Check, which combines hidden-state,
attention, and output-uncertainty signals for hallucination detection
~\citep{sriramanan2024llmcheck}. Our descriptors instead aggregate motion across
the \emph{token} axis at one fixed layer and can be evaluated on a response prefix.
Section~\ref{sec:baseline-comparisons} gives matched, question-disjoint comparisons
to both methods; we treat mathematical incorrectness as a response-level error
proxy rather than claiming that every mathematical error is a fabricated fact.

Recent work also has proposed geometric frameworks for understanding how LLMs reason.
\citet{zhou2025geometry} model reasoning as smooth flows in representation space,
using the velocity and Menger curvature of the trajectory to show that logical
structure, rather than surface semantics, governs the direction and magnitude
of these flows. Their focus is on interpretability. Our work takes a
complementary direction: we use geometric functionals of the hidden-state
trajectory, specifically the spectral effective dimension and kinematic
summaries, to predict task difficulty and solution correctness, connecting
trajectory geometry directly to downstream performance.

Lastly, 
\citet{prasad2026effective} show that effective reasoning strategies reduce the
intrinsic dimensionality of the learning objective, measured as the minimum number
of LoRA parameters needed to fine-tune a model to a given accuracy threshold on
GSM8K. They fix the model and vary the reasoning strategy, finding that lower
intrinsic dimensionality correlates strongly with better generalization. While
both their work and ours use notions of dimensionality to characterize reasoning,
the two measures are conceptually distinct. Their intrinsic dimension is a
property of the \emph{learning problem} induced by a reasoning strategy, it
requires fine-tuning experiments and measures how compressible a dataset of
reasoning chains is. Our effective dimension $\EffDim$ is a property of a
\emph{single inference trajectory}, it is computed from the covariance of
hidden states produced during one forward pass and requires no training. This
makes our measure applicable at inference time and enables per-instance
predictions of task difficulty and solution correctness.
\section{Problem Formulation}
\label{sec:formulation}

In this section we formalize Chain-of-Thought reasoning and develop a mathematical framework for characterizing its dynamics via Geometrical Functionals.
Consider a transformer language model with $L$ layers and hidden dimension $d$. Let $\vocab$ denote the finite vocabulary, and let $\vocab^* := \bigcup_{n=0}^{\infty} \vocab^n$ denote the set of all finite sequences over $\vocab$ (the Kleene star of $\mathcal{V}$). Let $\Delta(\vocab)$ denote the simplex of probability measures over $\vocab$. The model defines a map from finite token sequences to probability measures over the next token:
\begin{equation}
    \mu: \vocab^* \to \Delta(\vocab), \quad (x_1, \ldots, x_t) \mapsto \mu(\cdot \mid x_1, \ldots, x_t).
\end{equation}

As such for each layer $\ell \in \{1, \ldots, L\}$, the model also produces a hidden state representation in $\mathbb{R}^{d}$ where $d$ is the dimension of the latent representation, i.e.:
\begin{equation}
    f^{(\ell)}: \vocab^* \to \R^d, \quad (x_1, \ldots, x_t) \mapsto H_t^{(\ell)}\in \mathbb{R}^{d}.
\end{equation}

When the layer $\ell$ is fixed or clear from context, we write $H_t \equiv H_t^{(\ell)}$.
Given the distribution $\mu(\cdot \mid x_1, \ldots, x_t) \in \Delta(\vocab)$, the next token is selected according to a temperature parameter $T \geq 0$. In particular, at temperature $T > 0$, we sample from a tempered distribution:
\begin{equation}
    X_{t+1} \sim \mu_T(\cdot \mid x_1, \ldots, x_t), \quad \text{where } \mu_T(x) \propto \mu(x)^{1/T}.
\end{equation}

At temperature $T = 0$, the distribution concentrates on the mode:
\begin{equation}
    X_{t+1} = \arg\max_{x \in \vocab} \mu(x \mid x_1, \ldots, x_t).
\end{equation}
This distinction is fundamental: at $T = 0$, given a prompt, the generated sequence is unique; at $T > 0$, the same prompt yields a distribution over sequences.
\subsection{The Space of CoT Curves}
\label{sec:cot-curves}

At $T = 0$, token selection is deterministic. Given a prompt, there is exactly one generated sequence of tokens, this motivates the following definition  of the space of discrete curves, 
\begin{definition}
Fix a maximum sequence length $n \in \mathbb{N}$. The space of discrete curves of length $n$ is $\mathcal{C}_n := (\R^d)^n$. An element $\gamma \in \mathcal{C}_n$ is a tuple $\gamma = (h_1, \ldots, h_n)$ where $h_{i} \in \mathbb{R}^{d}$.
\end{definition}

We can embed $\mathcal{C}_{m} $ to $\mathcal{C}_n$ for $m \leq n$ by repeating the last element $n-m$ times (in practice, we do not apply this padding but instead work directly with variable-length trajectories).
Let $\mathcal{P} \subset \vocab^*$ denote the space of input prompts. At $T = 0$, the model defines a deterministic map from the set of prompts to the space of curves. In particular, we have:
\begin{equation}
    \mathcal{H}: \mathcal{P} \to \mathcal{C}_n, \quad P \mapsto (h_1, \ldots, h_m, h_m, \ldots, h_m),
\end{equation}
where $m$ is the generation length and the final state $H_m$ is repeated to fill length $n$ (stationary extension). As such we introduce the length $n$ chain of thought as an element in the space of discrete curves of length $n$. In other words, given a $\omega \in \vocab^*$
 the length-$n$ Chain-of-Thought curve is the element $\mathcal{H}(\omega) = (h_1(\omega), \ldots, h_n(\omega)) \in \mathcal{C}_n$ produced by generation from $\omega$, with stationary extension if necessary.
We characterize CoT curves through real or vector valued functionals, in particular:

\begin{definition}
\label{def:geometric-functional}
A vector-valued Geometric Functional is a function $\varphi: \mathcal{C}_n \to \mathbb{R}^{k}$ that extracts geometric properties of curves. The composition $\varphi \circ \mathcal{H}: \vocab^* \to \mathbb{R}^{k}$ characterizes how these properties vary across prompts.
\end{definition}

As an example, for a curve $\gamma = (h_1, \ldots, h_n) \in \mathcal{C}_n$, define the centered curve with elements $\bar{h}_t = h_t - \frac{1}{n}\sum_{i=1}^n h_i$ and its trajectory covariance matrix, defined as
$
    C(\gamma) := \frac{1}{n} \sum_{t=1}^{n} \bar{h}_t \bar{h}_t^\top = \frac{1}{n} \bar{H}^\top \bar{H} \in \R^{d \times d},
$
where $\bar{H} = [\bar{h}_1, \ldots, \bar{h}_n]^\top \in \R^{n \times d}$.
As such $C$ is a geometrical functional that takes a curve and produces an element in $\mathbb{R}^{d^2}$.

Let $\nu_1 \geq \nu_2 \geq \cdots \geq \nu_d \geq 0$ be the eigenvalues of $C(\gamma)$ ordered in a non-increasing order, then we introduce another such important geometrical functional as follows:
\begin{definition}
\label{def:effdim}
Let $\gamma \in \mathcal{C}_n$ be a curve with $\tr(C(\gamma)) > 0$.
For $\rho \in (0, 1]$, the Effective Dimension at $\rho$ is:
\begin{equation}
    \EffDim(\gamma) := \min\left\{k \in \{1,\ldots,d\} : \frac{\sum_{i=1}^{k} \nu_i}{\tr(C(\gamma))} \geq \rho \right\}.
\end{equation}
This is the minimum number of principal components needed to capture at least $\rho$ fraction of the total variance.
\end{definition}

The effective dimension measures the \emph{intrinsic dimensionality} of the reasoning trajectory. A low $\EffDim$ indicates the trajectory lies near a low-dimensional subspace (simple, structured reasoning), while high $\EffDim$ indicates the trajectory explores many directions (complex, multi-faceted reasoning).


Lastly, we introduce seven additional geometric functionals, which we use in Section~\ref{sec:exp-correctness} for correctness prediction.
Let $\gamma = (h_1, \ldots, h_n) \in \mathcal{C}_n$ be a trajectory, and let 
$P \in \mathbb{R}^{d \times k}$ denote the top-$k$ PCA basis fitted on the 
training set. Define the projected trajectory with elements $\tilde{h}_t = P^\top h_t \in 
\mathbb{R}^k$ for $t = 1, \ldots, n$, and let $\tilde{\gamma} = (\tilde{h}_1, 
\ldots, \tilde{h}_n)$ be the projected curve. Note that to prevent data leakage, when a fraction $\alpha \in (0,1]$ of 
the trajectory is observed, we restrict to the window $\tilde{\gamma}^{(\alpha)} = 
(\tilde{h}_1, \ldots, \tilde{h}_{\lfloor \alpha n \rfloor})$.
Noting that our framework is more general and one can extract many other meaningful functionals as needed, we define the following functionals $\varphi: \mathcal{C}_n \to \mathbb{R}^{5k+2}$ 
below.

\textbf{Examples} \emph{(Kinematic and Positional Geometric Functionals).}
\label{def:kinematic-functionals}
Given a projected trajectory $\tilde{\gamma} = (\tilde{h}_1, \ldots, 
\tilde{h}_m)$ for some $m \leq n$, define the velocity increments $\Delta_t := \tilde{h}_{t+1} - 
\tilde{h}_t \in \mathbb{R}^k$ for $1 \leq t \leq m-1$. The seven geometric 
functionals are:
\begin{enumerate}[leftmargin=2.0em, itemsep=0pt]
    \item Mean position, defined as $
        \mu(\tilde{\gamma}) 
        \;=\; \frac{1}{m}\sum_{t=1}^{m} \tilde{h}_t 
        \;\in\; \mathbb{R}^k.
    $

    \item Positional dispersion, the coordinate-wise standard 
    deviation:
    \begin{equation}
        \sigma(\tilde{\gamma}) 
        \;=\; \left(\frac{1}{m}\sum_{t=1}^{m}
        \bigl(\tilde{h}_t - \mu(\tilde{\gamma})\bigr)^{\odot 2}
        \right)^{\odot 1/2}
        \;\in\; \mathbb{R}^k,
    \end{equation}
    where $\odot$ denotes elementwise operations.

    \item Initial hidden state, the first token representation of the projected 
    trajectory, i.e., 
    $
        \tilde{h}_1 \;\in\; \mathbb{R}^k.
    $

    \item Final hidden state, the last token representation of the projected 
    trajectory, i.e., $
        \tilde{h}_m \;\in\; \mathbb{R}^k.
    $

    \item Mean velocity, the average of successive differences:
    \begin{equation}
        \bar{v}(\tilde{\gamma}) 
        \;=\; \frac{1}{m-1}\sum_{t=1}^{m-1} \Delta_t 
        \;\in\; \mathbb{R}^k.
    \end{equation}

    \item Mean speed,  the average step-wise Euclidean norm:
    \begin{equation}
        \bar{s}(\tilde{\gamma}) 
        \;=\; \frac{1}{m-1}\sum_{t=1}^{m-1} \|\Delta_t\|_2 
        \;\in\; \mathbb{R}.
    \end{equation}

    \item Speed dispersion, the standard deviation of step-wise 
    speeds:
    \begin{equation}
        \sigma_s(\tilde{\gamma}) 
        \;=\; \left(\frac{1}{m-1}\sum_{t=1}^{m-1}
        \bigl(\|\Delta_t\|_2 - \bar{s}(\tilde{\gamma})\bigr)^2
        \right)^{1/2}
        \;\in\; \mathbb{R}.
    \end{equation}

\end{enumerate}
The full feature vector is the concatenation
\begin{equation}
    \varphi(\tilde{\gamma}) 
    \;=\; \Bigl(\mu,\;\sigma,\;\tilde{h}_1,\;\tilde{h}_m,\;
    \bar{v},\;\bar{s},\;\sigma_s\Bigr) 
    \;\in\; \mathbb{R}^{5k+2}.
\end{equation}

Note that mean velocity telescopes to $\frac{1}{m-1}(\tilde{h}_m - \tilde{h}_1)$, making it a linear function of the already-included initial and final states. We retain it for completeness and its natural connection to mean speed and speed dispersion.
At temperature $T > 0$, the same prompt could yields different curves at each generation time. Our formulation can extend to this framework:
\begin{equation}
    \mathcal{H}: \vocab^* \to \Delta(\mathcal{C}_n),
\end{equation}
where $\Delta(\mathcal{C}_n)$ denotes distributions over curves. The induced distribution arises from the autoregressive measure:
\begin{equation}
    \mathbb{P}_{\mu_T}(v_1, \ldots, v_n) = \prod_{t=1}^{n} \mu_T(v_t \mid x_1, \ldots, x_k, v_1, \ldots, v_{t-1}).
\end{equation}

\section{Effective Dimension as Task Complexity}
\label{sec:effdim}
This section develops the theoretical core of the paper. We prove general spectral bounds for any covariance matrix, showing via a majorization argument that flat spectra maximize effective dimension (\S\ref{sec:spectral-bounds}) and provide perturbation stability. All proofs are in Appendix~\ref{app:proofs}.
\subsection{Spectral Bounds and the Role of Flatness}
\label{sec:spectral-bounds}
The effective dimension of any PSD matrix is controlled by its eigenvalue spread. These are purely linear-algebraic facts, independent of any dynamical model.

\begin{proposition}[Spectral Bounds]
\label{prop:spectral-bounds}
For any PSD matrix $C$ with eigenvalues $\nu_1 \geq \cdots \geq \nu_d \geq 0$, $\nu_1 > 0$:
\begin{equation}
    \lceil \rho \, \tr(C)/\nu_1 \rceil \;\leq\; \EffDim(C) \;\leq\; \lceil \rho \, \tr(C)/\nu_d \rceil,
\end{equation}
the upper bound requiring $\nu_d > 0$.
\end{proposition}

\begin{proof}[Proof sketch]
Let $r := \EffDim(C)$. The lower bound follows from $\rho \tr(C) \leq
\sum_{j \leq r} \nu_j \leq r \nu_1$, giving $r \geq \lceil \rho \tr(C)/\nu_1
\rceil$. The upper bound follows from $\sum_{j \leq r-1} \nu_j < \rho \tr(C)$
and $\nu_j \geq \nu_d$, giving $(r-1)\nu_d < \rho \tr(C)$, hence $r \leq
\lceil \rho \tr(C)/\nu_d \rceil$. The complete argument is in Appendix~\ref{app:proofs}.
\end{proof}

The ratio $\tr(C)/\nu_d$ is the spectral form of the hardness measure
$\Hard(\Lambda)=\tr(\Lambda)/\lambda_{\min}(\Lambda)$ in Eq.~3.11 of
\citet{javanmard2025understanding}; they apply it to a feature covariance rather
than a trajectory covariance. The upper bound shows it controls $\EffDim$, but
\emph{loosely}. When the spectrum is flat, the bounds coincide at
$\lceil \rho d\rceil$. Flatness is extremal in a stronger sense, captured by
majorization.

\begin{definition}
\label{def:majorization}
For $x, y \in \R^d$, we say $x$ is majorized by $y$, written $x \prec y$, if
\begin{equation}
\sum_{i=1}^k x_{[i]} \;\leq\; \sum_{i=1}^k y_{[i]} \quad \text{for all } k = 1, \ldots, d-1, \quad \text{and} \quad \sum_{i=1}^d x_{[i]} = \sum_{i=1}^d y_{[i]},
\end{equation}
where $x_{[1]} \geq \cdots \geq x_{[d]}$ is the decreasing rearrangement of $x$. Informally: $x$'s mass is more ``spread out'' than $y$'s, but both have the same total \cite{marshall1979inequalities}.
\end{definition}

Majorization gives us a precise way to compare how ``peaked'' two spectra are. We use it to show that the flat spectrum $\nu^\star$ is the most spread-out among all spectra with the same trace, and achieves the highest effective dimension. More formally we have:

\begin{proposition}[Flat Spectrum Maximizes Effective Dimension]
\label{prop:flat-maximal}
Let $C$ be PSD with $\tr(C)>0$, eigenvalues $\nu \in \R^d$. Let $\nu^\star := (\tr(C)/d, \ldots, \tr(C)/d)$, the flat spectrum with the same total. Then:
\begin{enumerate}
    \item $\nu^\star \prec \nu$: the flat spectrum is majorized by any other spectrum with the same total.
    \item $\EffDim(C) \leq \EffDim(C^\star) = \lceil \rho d\rceil$.
\end{enumerate}
\end{proposition}

\begin{proof}[Proof sketch]
We first show $\nu^\star \prec \nu$: since both vectors have total
$\tr(C)$, this reduces to $\sum_{i \leq k} \nu_i \geq k\tr(C)/d$, i.e.,
the top-$k$ average is at least the overall mean. If not, then $\nu_k$
(and hence every $\nu_j$ with $j > k$, by decreasing order) is also
below $\tr(C)/d$, so $\tr(C) = \sum_i \nu_i < \tr(C)$, a contradiction. The second one follows by definition and using part 1.  The complete argument is in Appendix~\ref{app:proofs}. 
\end{proof}
\subsection{Stability under small Perturbation}
\label{sec:stability}
The following framework-independent and purely linear-algebraic result shows that if two covariance matrices are $\epsilon$-close in operator norm, their effective dimensions are also close.

\begin{theorem}
\label{thm:stability}
Let $C, \widehat C$ be PSD with $\|\widehat C - C\|_{\mathrm{op}} \leq \epsilon$ and $\tr(C) > 0$. Assume $d\epsilon \leq \tr(C)/2$. Define $F(j) := \sum_{i\leq j}\nu_i/\tr(C)$ where $\nu_i$ are the decreasingly-sorted eigenvalues of $C$. Then
\begin{equation}
\label{eq:stability}
\bigl|\EffDim(\widehat C) - \EffDim(C)\bigr|
\;\leq\;
\#\!\left\{ j \in \{1,\ldots,d\} : |F(j) - \rho| \leq \tfrac{4d\epsilon}{\tr(C)}\right\}.
\end{equation}
In particular, if the cumulative mass function $F$ crosses level $\rho$ transversally (i.e., $F$ has no index within distance $4d\epsilon/\tr(C)$ of $\rho$), then $\EffDim(\widehat C) = \EffDim(C)$.
\end{theorem}
\begin{proof}[Proof sketch]
The argument proceeds in three steps. First, recall \emph{Weyl's
inequality}: for symmetric matrices $A, B \in \mathbb{R}^{d \times d}$
with eigenvalues sorted in decreasing order,
$|\lambda_k(A + B) - \lambda_k(A)| \leq \|B\|_{\mathrm{op}}$ for every
$k$. Applied with $A = C$ and $B = \widehat{C} - C$, this gives
$|\hat\nu_k - \nu_k| \leq \epsilon$ for all $k$: each eigenvalue of
$\widehat{C}$ is within $\epsilon$ of the corresponding eigenvalue of
$C$. Summing across $k$, we also get $|\tr(\widehat{C}) - \tr(C)| \leq
d\epsilon$, and similarly the partial sums $S_j := \sum_{i \leq j}
\nu_i$ and $\widehat{S}_j$ differ by at most $d\epsilon$.

Second, we propagate this to the cumulative-mass function
$F(j) := S_j / \tr(C)$. Using the algebraic identity $\widehat{S}_j
\tr(C) - S_j \tr(\widehat{C}) = (\widehat{S}_j - S_j)\tr(C) + S_j
(\tr(C) - \tr(\widehat{C}))$ and bounding each piece via the triangle
inequality, we obtain $|\widehat{F}(j) - F(j)| \leq 4d\epsilon / \tr(C)
=: \delta$.

Finally, since $\EffDim$ is the first index at which $F$ reaches
$\rho$, and $|F(j) - \widehat{F}(j)| \leq \delta$ everywhere, the two
effective dimensions can only disagree at indices where $F(j)$ lies
within $\delta$ of $\rho$: away from this band, $F$ and $\widehat{F}$
agree on whether the threshold has been crossed. Counting such indices
yields \eqref{eq:stability}. In particular, if $F$ jumps past $\rho$
transversally at a single index, with no $j$ satisfying $|F(j) -
\rho| \leq \delta$, then $\EffDim(\widehat{C}) = \EffDim(C)$
exactly. The complete argument is in Appendix~\ref{app:proofs}.
\end{proof}

\section{Experiments}
\label{sec:experiments}
\subsection{Experimental Setup}
\paragraph{Model.}
We use Qwen2.5-Instruct~\citep{qwen2025qwen25technicalreport} at three scales: 0.5B, 1.5B, and
7B parameters. The 0.5B model has 24 layers and the 1.5B and 7B models have 28
layers. We fixed the primary analysis layer at
75\% relative depth: layer 18 for 0.5B and layer 21 for 1.5B and 7B (zero-based
indices). We additionally
evaluate two neighboring depths per model as a sensitivity analysis
(Appendix~\ref{app:robustness}).

\paragraph{Dataset.}
We use a frozen, balanced set of 86 questions from the MATH500
dataset~\citep{lightman_let_2024} (a 500-problem subset of MATH~\citep{hendrycks2021measuring}): 43 difficulty-1 questions (\texttt{easy})
and 43 difficulty-5 questions (\texttt{hard}), spanning all seven MATH subject
categories. The identical question set is used at every model scale.

\paragraph{Trajectory Collection.}
For each problem and model, we generate three reasoning trajectories at
temperature $T = 0.7$ using the medium chain-of-thought prompt in the appendix,
for 258 trajectories per model. Each trajectory is generated autoregressively
with a maximum of 1,200 tokens. We extract $h_t^{(\ell)}$ at every generated token
position for the preregistered primary layer and two neighboring layers; the
all-layer baselines use either a dedicated all-layer collection (0.5B CoE) or
teacher-forcing of the same saved responses (7B CoE and LLM-Check). Correctness
is determined by a symbolic answer checker combining SymPy
expression matching and string normalization against the ground-truth boxed
answer.

\paragraph{Features for Correctness Prediction.}
For each trajectory, we project $H^{(\ell)}$ onto the top 15 PCA components
fitted on the training set, then extract seven kinematic and positional features
from the windowed portion of the projected trajectory (which are defined in section \ref{sec:cot-curves}): mean position, positional
dispersion (standard deviation), initial hidden state, final hidden state of the truncated trajectories, mean
velocity (mean of successive differences), mean speed (mean of step-wise norms),
and speed dispersion. These are standard first and second-order moment descriptors of a Euclidean curve---its location, spread, endpoints, and rate and variability of motion---and were chosen without domain-specific tuning. Features are standardized before
classification.

\paragraph{Evaluation Protocol.}
For correctness prediction, we use a stratified question-level 80/20 train/test
split repeated over 10 random seeds, ensuring that all trajectories from a given
question appear entirely in train or entirely in test. PCA is fit using training
questions only. We report mean test AUROC and 95\% split-bootstrap confidence
intervals. For the matched CoE and LLM-Check comparisons, we use five disjoint
question folds so every question is tested exactly once, and paired stratified
question-cluster bootstrap intervals. The 0.5B CoE comparison instead uses a
dedicated five-generation all-layer collection and 10 repeated question-disjoint
80/20 splits. For difficulty, each question is the unit
of inference: we average its three trajectories, evaluate leave-one-question-out,
and obtain confidence intervals and permutation $p$-values by resampling or
permuting questions rather than trajectories.

\subsection{Experiment 1: Correctness Prediction from Trajectory Geometry}
\label{sec:exp-correctness}
We ask whether the geometry of a reasoning trajectory predicts whether it will reach a correct answer, and how early in generation this signal emerges.

\begin{figure}[h]
    \centering
    \includegraphics[width=\textwidth]{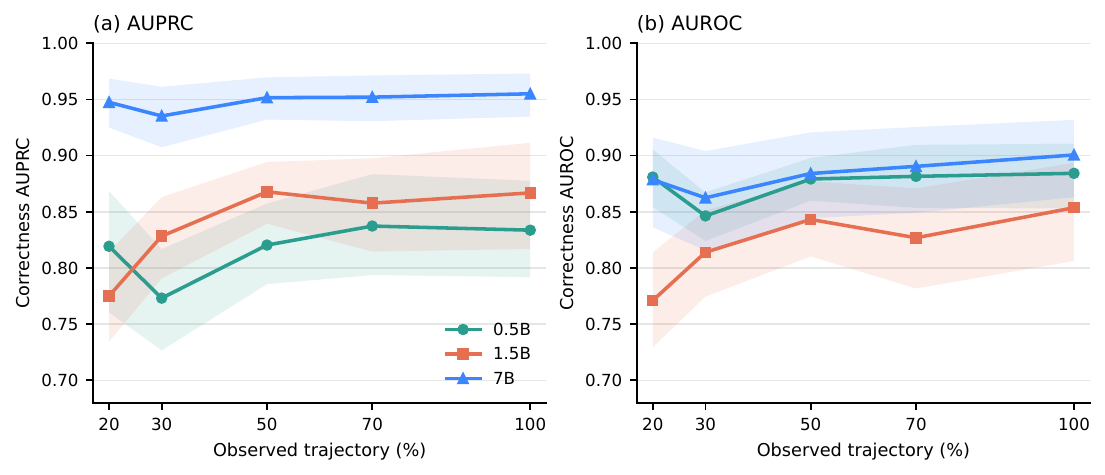}
    \caption{Correctness prediction across model scales as a function of the
    observed trajectory fraction. \textbf{(a)} AUPRC. \textbf{(b)} AUROC.
    Curves show the preregistered single-layer LR; shaded regions are 95\%
    split-bootstrap intervals over 10 repeated, question-disjoint 80/20 splits.
    The positive class is a correct response.}
    \label{fig:correctness}
\end{figure}

\paragraph{Cross-question generalization.}
Figure~\ref{fig:correctness} reports the preregistered single-layer LR across
five observation fractions. At 20\%, AUROC is $0.88$
$[0.85,0.91]$ for 0.5B, $0.77$ $[0.73,0.81]$ for 1.5B, and $0.88$
$[0.84,0.92]$ for 7B. Full-response AUROC is $0.88$ $[0.85,0.91]$, $0.85$
$[0.81,0.89]$, and $0.90$ $[0.86,0.93]$, respectively. Thus, the early signal
replicates at all three scales. Individual per-question AUROCs are only descriptive because each question
has three generations and most questions contain a single outcome class; the
primary grouped analysis retains all 86 questions. Full AUROC and AUPRC values
with confidence intervals are reported in Appendix~\ref{app:correctness-sweep}.
The result is also stable under strict fixed-token observation: at 64 tokens,
AUROC is $0.88$, $0.80$, and $0.84$ for 0.5B, 1.5B, and 7B, respectively
(Appendix~\ref{app:fixed-token-correctness}).

\subsection{Matched Comparisons to CoE and LLM-Check}
\label{sec:baseline-comparisons}
Table~\ref{tab:matched-baselines} compares methods within two panels so every
baseline is paired with geometry on the same responses and splits. For
CoE~\citep{wang2025coe}, we compute both reported variants from the full layer
stack; values below 100\% are matched-prefix adaptations, whereas 100\% is the
method-faithful completed-response comparison. The 0.5B CoE panel uses a
dedicated all-layer collection of five generations per question (430 responses);
the other comparisons use the three-generation collections (258 responses). For
LLM-Check~\citep{sriramanan2024llmcheck}, we teacher-force the same response or
prefix and compute its hidden, attention, perplexity, logit-entropy,
and window-entropy signals. We report hidden and attention scores both at the
same fixed layer as ours and with LLM-Check's layer selected using training
questions only. Geometry uses the same fixed relative depth as in our main
analysis. LLM-Check and 7B CoE use five grouped out-of-fold splits; 0.5B CoE
uses 10 repeated question-disjoint 80/20 splits. Full fraction sweeps and
confidence intervals appear in Appendix~\ref{app:baselines}. Because the grouped
out-of-fold comparisons use one prediction per trajectory rather than averaging
overlapping repeated holdouts, their geometry estimates differ slightly from
Figure~\ref{fig:correctness}.

\begin{table}[h]
\centering
\small
\caption{Correctness comparisons (mean test AUROC). Within each panel and column,
geometry and the baselines use the same responses and test splits; best
performance is bold. Panel A uses 258 responses per model. Panel B uses an
independent 430-response all-layer collection at 0.5B and 258 responses at 7B.}
\label{tab:matched-baselines}
\resizebox{\textwidth}{!}{%
\begin{tabular}{lcccc}
\toprule
Method & 0.5B, 20\% & 0.5B, 100\% & 7B, 20\% & 7B, 100\% \\
\midrule
\multicolumn{5}{l}{\textit{Panel A: LLM-Check}} \\
Geometry (ours)         & \textbf{0.82} & \textbf{0.86} & \textbf{0.90} & \textbf{0.91} \\
LLM-Check hidden        & 0.81 & 0.81 & 0.89 & 0.83 \\
LLM-Check hidden (selected) & 0.79 & 0.83 & 0.89 & 0.84 \\
LLM-Check attention     & 0.59 & 0.42 & 0.77 & 0.76 \\
LLM-Check attention (selected) & 0.62 & 0.76 & 0.83 & 0.72 \\
LLM-Check perplexity    & 0.63 & 0.78 & 0.55 & 0.55 \\
LLM-Check logit entropy & 0.62 & 0.78 & 0.44 & 0.62 \\
LLM-Check window entropy& 0.77 & 0.83 & 0.73 & 0.74 \\
\midrule
\multicolumn{5}{l}{\textit{Panel B: Chain-of-Embedding}} \\
Geometry (ours)         & \textbf{0.84} & \textbf{0.87} & \textbf{0.90} & \textbf{0.91} \\
CoE-R                   & 0.33 & 0.51 & 0.53 & 0.75 \\
CoE-C                   & 0.60 & 0.75 & 0.56 & 0.76 \\
\bottomrule
\end{tabular}%
}
\end{table}

Our single-layer token-axis geometry has the highest mean AUROC at both early
and full observation in both panels and at both evaluated scales. At 0.5B it
reaches $0.84$ and $0.87$ on the all-layer CoE collection, versus $0.60$ and
$0.75$ for CoE-C. At 7B and 100\%, it reaches $0.91$ versus $0.84$ for the
strongest LLM-Check signal and $0.76$ for CoE-C.
The comparison also clarifies the methodological distinction: CoE aggregates
across layers and is designed for a completed response, whereas our method
aggregates across tokens at one layer and is already strong at 20\%.
\subsection{Experiment 2: Difficulty Prediction via Effective Dimension}
\label{sec:exp-difficulty}

We ask whether the effective dimension $\EffDim$ of the hidden-state trajectory
can predict whether a problem is easy or hard (MATH500 difficulty of $1$ vs. $5$)
without using the ground-truth answer or response text as a feature. At each
model's preregistered 75\%-depth layer, we compute
$\EffDim(H^{(\ell)},\rho)$ for $\rho \in \{0.90,0.95,0.99\}$ and average the
three trajectories for each question. We then run leave-one-question-out
evaluation at the question level.

\paragraph{Results.}
Table~\ref{tab:multiscale-difficulty} and Figure~\ref{fig:main-results} show that
the easy--hard separation strengthens with model scale. Difficulty AUROC rises
from $0.90$ at 0.5B to $0.94$ at 1.5B and $0.96$ at 7B. The signal is already
present at 20\% of the trajectory, where AUROC is $0.91$,
$0.93$, and $0.95$, respectively. Importantly, the 7B model
solves 48.06\% of hard trajectories, so the result persists when ``hard'' is not
simply beyond the model's capability. Inference is question-level: a 10,000-draw
permutation test on question-mean $d_{0.95}$ gives $p<10^{-4}$ at every scale
(the Monte Carlo resolution).

\begin{table}[t]
\centering
\small
\caption{Difficulty results at the preregistered layer. Accuracy is the fraction
of correct trajectories; $\Delta d_{0.95}$ is hard minus easy. AUROC intervals
are 95\% question-bootstrap confidence intervals.}
\label{tab:multiscale-difficulty}
\resizebox{\textwidth}{!}{%
\begin{tabular}{lrrrrrr}
\toprule
Model & Easy acc. & Hard acc. & Easy $d_{0.95}$ & Hard $d_{0.95}$ & $\Delta d_{0.95}$ & Difficulty AUROC \\
\midrule
Qwen2.5-0.5B & 63.57\% & 11.63\% & 133.27 & 185.73 & 52.46 & 0.90 $[0.82,0.97]$ \\
Qwen2.5-1.5B & 85.27\% & 17.05\% & 126.64 & 205.82 & 79.18 & 0.94 $[0.87,0.98]$ \\
Qwen2.5-7B   & 96.12\% & 48.06\% & 155.48 & 284.51 & 129.03 & 0.96 $[0.92,0.99]$ \\
\bottomrule
\end{tabular}%
}
\end{table}

\begin{figure}[t]
    \centering
    \includegraphics[width=\textwidth]{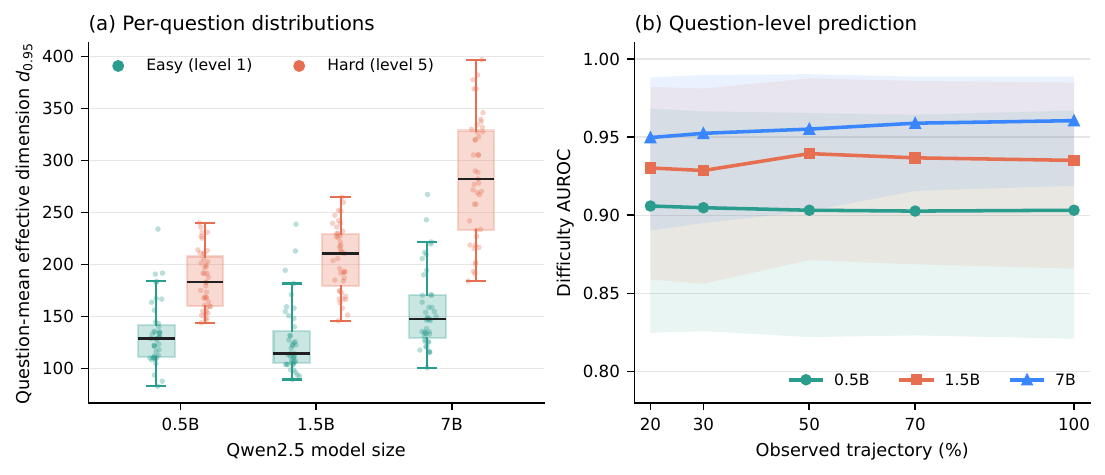}
    \caption{Effective dimension predicts task difficulty across model scales.
    \textbf{(a)} All question-mean $d_{0.95}$ values at the preregistered layer
    (43 easy and 43 hard questions per model); boxes show quartiles and medians.
    \textbf{(b)} Leave-one-question-out difficulty AUROC across observed
    trajectory fractions; shaded regions are 95\% question-bootstrap confidence
    intervals. Fractions are defined relative to the eventual response length.}
    \label{fig:main-results}
\end{figure}
\section{Conclusion}
\label{sec:conclusion}

We study the geometry of chain-of-thought reasoning trajectories in transformer
hidden state space. Formalizing each reasoning chain as a discrete curve in
$\R^d$, we introduced the effective dimension $\EffDim$ as a spectral measure of
trajectory complexity and empirically found greater effective dimension for
harder tasks. Across Qwen2.5-0.5B, 1.5B,
and 7B, $\EffDim$ predicts problem difficulty with AUROC $0.90$--$0.96$ via
leave-one-question-out evaluation, and the easy--hard separation grows with
model scale. For correctness prediction, seven kinematic features achieve AUROC
$0.77$--$0.88$ from only the first $20\%$ of generated tokens under
question-disjoint splits. At 7B, the matched full-response AUROC is $0.91$, versus
$0.84$ for LLM-Check and $0.76$ for CoE-C. Together, these results establish
trajectory geometry as a practical window into both task hardness and solution
quality.
\smallskip

\noindent\textbf{Limitations.}
Our experiments cover three sizes from one model family and one mathematical
benchmark, so generalization to other architectures and domains remains open. Our theoretical
analysis connects spectral flatness to effective dimension but does not fully
explain how task difficulty induces this flatness in a trained transformer.
Finally, our correctness prediction is evaluated in terms of AUC; translating
this into concrete early-stopping or best-of-$n$ gains is left for future work.

\newpage
\bibliographystyle{plainnat}   
\bibliography{refs}            

\appendix

\section{Appendix}

\subsection{Broader Impacts}
\label{app:broader_impacts}

This work studies how the geometry of a language model's internal reasoning trajectories encodes task difficulty and solution quality. On the positive side, our framework could inform more efficient inference strategies: early correctness signals could reduce wasted compute in best-of-$n$ sampling, and difficulty estimates could guide adaptive test-time budget allocation. The spectral and kinematic features we introduce are lightweight and model-agnostic, making them potentially applicable across architectures without requiring access to model weights beyond hidden states. We acknowledge that trajectory-based early selection is imperfect: promising reasoning paths may be abandoned due to unfavorable early geometry, while some trajectories that appear geometrically favorable may still fail to reach a correct answer. Nevertheless, even modest improvements in the ability to identify and continue good reasoning trajectories, or to prune poor ones earlier, can yield meaningful gains in efficiency and reliability at scale, where inference compute costs are substantial and the cumulative benefit of small per-query improvements is large. On the negative side, a system that predicts solution correctness before generation completes could in principle be used to preferentially surface or suppress certain outputs in ways that are not transparent to users. More broadly, tools for probing model internals carry dual-use risk: the same interpretability techniques that help researchers understand failure modes could also inform adversarial probing of deployed systems. We do not foresee immediate harmful applications of this specific work, but encourage the community to consider access-control norms around hidden-state extraction from production models. Our experiments use only publicly available datasets and open-weight models.

\subsection{Proofs}
\label{app:proofs}
In the following proofs, for simplicity of the notations we used $T := \tr(C)$.

\subsubsection{Proof of Proposition~\ref{prop:spectral-bounds}}

\begin{proof}
Let $r := \EffDim(C)$. By definition of $\EffDim$, $r$ is the smallest integer with $\sum_{j=1}^{r}\nu_j \geq \rho T$, so we have
\begin{equation}
\sum_{j=1}^{r}\nu_j \;\geq\; \rho T \quad\text{and}\quad \sum_{j=1}^{r-1}\nu_j \;<\; \rho T.
\label{eq:effdim-def}
\end{equation}

\smallskip
\textbf{Lower bound.} Since the eigenvalues are sorted in decreasing order, $\nu_j \leq \nu_1$ for all $j$. Summing over $j=1,\ldots,r$:
\[
\rho T \;\leq\; \sum_{j=1}^{r}\nu_j \;\leq\; r \cdot \nu_1.
\]
Dividing by $\nu_1 > 0$ gives $r \geq \rho T/\nu_1$. Since $r$ is a positive integer, $r \geq \lceil \rho T/\nu_1\rceil$.

\smallskip
\textbf{Upper bound.} Since $\nu_j \geq \nu_d$ for all $j \leq r-1$, summing:
\[
(r-1) \cdot \nu_d \;\leq\; \sum_{j=1}^{r-1}\nu_j \;<\; \rho T,
\]
where the strict inequality is from \eqref{eq:effdim-def}. Hence $r-1 < \rho T/\nu_d$, and integrality gives $r \leq \lceil \rho T / \nu_d\rceil$.
\end{proof}

\subsubsection{Proof of Proposition~\ref{prop:flat-maximal}}

\begin{proof}
\textbf{(1) $\nu^\star \prec \nu$.} Both vectors have the same total $T$, so the equality condition in Def.~\ref{def:majorization} holds. We must show $\sum_{i=1}^k \nu^\star_i \leq \sum_{i=1}^k \nu_i$ for every $k$, i.e., $kT/d \leq \sum_{i=1}^k \nu_i$. Equivalently, we must show the top-$k$ average of $\nu$ is at least the overall average:
\[
\frac{1}{k}\sum_{i=1}^k \nu_i \;\geq\; \frac{T}{d}.
\]

Suppose for contradiction that $\frac{1}{k}\sum_{i=1}^k \nu_i < T/d$. Since $\nu$ is sorted decreasingly, $\nu_k \leq \frac{1}{k}\sum_{i=1}^k \nu_i < T/d$ and $\nu_j \leq \nu_k < T/d$ for all $j \geq k$. Then
\[
T \;=\; \underbrace{\sum_{i=1}^k \nu_i}_{< kT/d} \;+\; \underbrace{\sum_{i=k+1}^d \nu_i}_{< (d-k)T/d} \;<\; \frac{kT}{d} + \frac{(d-k)T}{d} \;=\; T,
\]
a contradiction. Hence the top-$k$ average is at least $T/d$, and $\nu^\star \prec \nu$.

\smallskip
\textbf{(2) $\EffDim$ is maximized at the flat spectrum.} Let $r^\star := \EffDim(C^\star) = \lceil \rho d\rceil$ (directly, since $\sum_{i=1}^{r^\star} \nu^\star_i = r^\star T/d \geq \rho T$ iff $r^\star \geq \rho d$). Applying part (1) at $k = r^\star$:
\[
\sum_{i=1}^{r^\star} \nu_i \;\geq\; \sum_{i=1}^{r^\star} \nu^\star_i \;\geq\; \rho T.
\]
By definition of $\EffDim$ as the \emph{smallest} such $r$, $\EffDim(C) \leq r^\star = \EffDim(C^\star)$.

\smallskip
\textbf{Equality condition.} Equality throughout requires $\nu^\star \prec \nu$ to be an equality of partial sums at every $k$, forcing $\nu_i = T/d$ for all $i$, i.e., $\nu = \nu^\star$.
\end{proof}

\subsubsection{Proof of Theorem~\ref{thm:stability}}
\begin{proof}

Before we prove this, we recall Weyl inequality which states,
for any symmetric matrices $A, B \in \mathbb{R}^{d \times d}$ with eigenvalues
sorted in decreasing order,
\begin{equation*}
\bigl|\lambda_k(A + B) - \lambda_k(A)\bigr| \;\leq\; \|B\|_{\mathrm{op}}
\qquad \text{for all } k = 1, \ldots, d.
\end{equation*}
where $\|B\|_{\mathrm{op}} := \text{max}(|\lambda_{1}(B)|,|\lambda_d(B)|)$.
In particular, taking $A = C$ and $B = \widehat{C} - C$ gives
$|\hat\nu_k - \nu_k| \leq \|\widehat{C} - C\|_{\mathrm{op}}$ which we use shortly. Let $\nu_1 \geq \cdots \geq \nu_d$ and $\hat\nu_1 \geq \cdots \geq \hat\nu_d$ be the decreasingly-sorted eigenvalues of $C$ and $\widehat C$.
Weyl's inequality for Hermitian matrices gives $|\hat\nu_j - \nu_j| \leq \|\widehat C - C\|_{\mathrm{op}} \leq \epsilon$ for every $j$. Summing: $|\tr(\widehat C) - \tr(C)| = |\sum_j (\hat\nu_j - \nu_j)| \leq d\epsilon$.
Now, let $S_j := \sum_{i=1}^j \nu_i$ and $\widehat S_j := \sum_{i=1}^j \hat\nu_i$. By Weyl again, $|\widehat S_j - S_j| \leq j\epsilon \leq d\epsilon$.

Define $F(j) := S_j/\tr(C)$ and $\widehat F(j) := \widehat S_j/\tr(\widehat C)$. We compute
\[
\bigl|\widehat F(j) - F(j)\bigr| \;=\; \left|\frac{\widehat S_j}{\tr(\widehat C)} - \frac{S_j}{\tr(C)}\right| \;=\; \left|\frac{\widehat S_j \tr(C) - S_j \tr(\widehat C)}{\tr(\widehat C)\tr(C)}\right|.
\]
Using $\widehat S_j \tr(C) - S_j \tr(\widehat C) = (\widehat S_j - S_j)\tr(C) + S_j(\tr(C) - \tr(\widehat C))$ and the triangle inequality:
\[
\bigl|\widehat F(j) - F(j)\bigr| \;\leq\; \frac{|\widehat S_j - S_j|}{\tr(\widehat C)} + \frac{S_j \cdot |\tr(C) - \tr(\widehat C)|}{\tr(\widehat C)\tr(C)} \;\leq\; \frac{d\epsilon}{\tr(\widehat C)} + \frac{d\epsilon \cdot S_j}{\tr(\widehat C)\tr(C)}.
\]

Since $S_j \leq \tr(C)$, the second term is at most $d\epsilon/\tr(\widehat C)$. Also $\tr(\widehat C) \geq \tr(C) - d\epsilon \geq \tr(C)/2$ (using the assumption $d\epsilon \leq \tr(C)/2$). Therefore
\[
\bigl|\widehat F(j) - F(j)\bigr| \;\leq\; \frac{2 d\epsilon}{\tr(\widehat C)} \;\leq\; \frac{4 d\epsilon}{\tr(C)}.
\]

By definition, $\EffDim(C) = \min\{j : F(j) \geq \rho\}$ and similarly for $\widehat C$. If $F$ and $\widehat F$ differ by at most $\delta := 4d\epsilon/\tr(C)$ at every index, then the two threshold-crossings can differ only at indices where $F$ is within $\delta$ of $\rho$. (Formally: if $|F(j)-\rho| > \delta$, then $F(j) > \rho + \delta \Rightarrow \widehat F(j) > \rho$, or $F(j) < \rho - \delta \Rightarrow \widehat F(j) < \rho$; either way $F$ and $\widehat F$ agree on ``reached threshold at $j$'' at that index.) So the difference $|\EffDim(\widehat C) - \EffDim(C)|$ is bounded by the number of indices $j$ within $|F(j) - \rho| \leq \delta$.

\end{proof}

\subsection{Experimental Details}
\label{app:experimental}
\paragraph{Hardware.} Trajectory collection and teacher-forced baseline extraction
were run on NVIDIA RTX 5060 Ti (16\,GB) and A100 (40\,GB) GPUs. Effective-dimension
and classifier analyses were run on CPU after collection.

\paragraph{Hyperparameters.}
\begin{itemize}
    \item Temperature: $T = 0.7$
    \item Maximum generated tokens: 1,200
    \item Runs per question: 3
    \item Minimum tokens for valid trajectory: 30
    \item $\rho$ values: $\{0.90, 0.95, 0.99\}$
\end{itemize}

\paragraph{Difficulty Prediction Classifier Details.}
\begin{itemize}
    \item \textbf{Logistic Regression}: $C = 1.0$, max iterations = 1000
    \item \textbf{MLP}: Hidden layers $(32, 16)$, early stopping with 15\% validation, $\ell_2$ regularization $\alpha = 0.01$
\end{itemize}

\paragraph{Correctness Prediction Classifier Details.}
\begin{itemize}
    \item \textbf{Logistic Regression}: $C = 0.1$, max iterations = 1000, balanced class weights
    \item \textbf{MLP}: Hidden layers $(64, 32)$, early stopping with 15\% validation, $\ell_2$ regularization $\alpha = 0.01$, balanced sample weights
    \item \textbf{GRU}: 2-layer, hidden dim 64, dropout 0.3, Adam optimizer 
(lr=$1e-3$, weight decay=$1e-4$), 30 epochs, batch size 32, 
BCE loss with positive class weighting, gradient clipping norm $1.0$
\end{itemize}

\subsection{Robustness and Additional Results}
\label{app:robustness}

\paragraph{Layer sensitivity.}
The primary layers were fixed at a common 75\% relative depth before analysis.
Table~\ref{tab:layer-sensitivity} shows that using nearby depths gives essentially
the same difficulty AUROC, ruling out a best-of-layers explanation for the
headline result.

\begin{table}[h]
\centering
\small
\caption{Difficulty AUROC $[95\%$ question-bootstrap CI$]$ at three fixed depths.}
\label{tab:layer-sensitivity}
\begin{tabular}{lccc}
\toprule
Model & Approx. 50\% depth & 75\% depth (primary) & Approx. 89\% depth \\
\midrule
Qwen2.5-0.5B & L12: 0.91 $[0.83,0.96]$ & L18: 0.90 $[0.82,0.97]$ & L21: 0.90 $[0.82,0.96]$ \\
Qwen2.5-1.5B & L14: 0.93 $[0.87,0.98]$ & L21: 0.94 $[0.87,0.98]$ & L25: 0.94 $[0.87,0.98]$ \\
Qwen2.5-7B   & L14: 0.96 $[0.92,0.99]$ & L21: 0.96 $[0.92,0.99]$ & L25: 0.96 $[0.92,0.99]$ \\
\bottomrule
\end{tabular}
\end{table}

\paragraph{Truncation.}
Hard responses hit the 1,200-token cap more often, but excluding every truncated
trajectory leaves a large question-level separation at all scales
(Table~\ref{tab:truncation}). Thus, the effect is not created by capped responses.

\begin{table}[h]
\centering
\small
\caption{Truncation check. The final two columns use only non-truncated
trajectories; intervals and permutation tests resample questions.}
\label{tab:truncation}
\begin{tabular}{lrrrr}
\toprule
Model & Easy trunc. & Hard trunc. & $\Delta d_{0.95}$ $[95\%$ CI$]$ & Permutation $p$ \\
\midrule
Qwen2.5-0.5B & 2.33\% & 12.40\% & 52.55 $[40.54,64.50]$ & $<10^{-4}$ \\
Qwen2.5-1.5B & 1.55\% & 10.85\% & 75.70 $[62.37,88.07]$ & $<10^{-4}$ \\
Qwen2.5-7B   & 0.00\% & 16.28\% & 118.60 $[98.07,138.05]$ & $<10^{-4}$ \\
\bottomrule
\end{tabular}
\end{table}

\paragraph{Alternative spectral functional.}
The participation ratio also separates hard from easy questions: the hard-minus-
easy difference is $2.45$ ($p=0.02$) at 0.5B, $4.13$ ($p<10^{-4}$) at 1.5B,
and $5.37$ ($p=2.00\times10^{-4}$) at 7B. This check uses the same primary layers
and question-level permutation test as $d_{0.95}$.

\subsubsection{Full Correctness Fraction Sweep}
\label{app:correctness-sweep}

Tables~\ref{tab:correctness-auc-sweep} and
\ref{tab:correctness-auprc-sweep} give the values plotted in
Figure~\ref{fig:correctness}. Each entry is the mean test metric and 95\%
split-bootstrap interval over the same 10 repeated question-disjoint 80/20
splits. The positive class is a correct response.

\begin{table}[h]
\centering
\small
\caption{Correctness AUROC across observation fractions.}
\label{tab:correctness-auc-sweep}
\begin{tabular}{lccc}
\toprule
Obs. & Qwen2.5-0.5B & Qwen2.5-1.5B & Qwen2.5-7B \\
\midrule
20\%  & 0.88 $[0.85,0.91]$ & 0.77 $[0.73,0.81]$ & 0.88 $[0.84,0.92]$ \\
30\%  & 0.85 $[0.82,0.87]$ & 0.81 $[0.77,0.85]$ & 0.86 $[0.82,0.90]$ \\
50\%  & 0.88 $[0.86,0.90]$ & 0.84 $[0.81,0.88]$ & 0.88 $[0.84,0.92]$ \\
70\%  & 0.88 $[0.85,0.91]$ & 0.83 $[0.78,0.87]$ & 0.89 $[0.85,0.93]$ \\
100\% & 0.88 $[0.85,0.91]$ & 0.85 $[0.81,0.89]$ & 0.90 $[0.86,0.93]$ \\
\bottomrule
\end{tabular}
\end{table}

\begin{table}[h]
\centering
\small
\caption{Correctness AUPRC across observation fractions.}
\label{tab:correctness-auprc-sweep}
\begin{tabular}{lccc}
\toprule
Obs. & Qwen2.5-0.5B & Qwen2.5-1.5B & Qwen2.5-7B \\
\midrule
20\%  & 0.82 $[0.76,0.87]$ & 0.78 $[0.73,0.81]$ & 0.95 $[0.93,0.97]$ \\
30\%  & 0.77 $[0.73,0.82]$ & 0.83 $[0.79,0.86]$ & 0.94 $[0.91,0.96]$ \\
50\%  & 0.82 $[0.79,0.86]$ & 0.87 $[0.84,0.89]$ & 0.95 $[0.93,0.97]$ \\
70\%  & 0.84 $[0.79,0.88]$ & 0.86 $[0.81,0.90]$ & 0.95 $[0.93,0.97]$ \\
100\% & 0.83 $[0.79,0.88]$ & 0.87 $[0.82,0.91]$ & 0.96 $[0.93,0.97]$ \\
\bottomrule
\end{tabular}
\end{table}

\subsubsection{Fixed-Token Correctness Robustness}
\label{app:fixed-token-correctness}

Figure~\ref{fig:fixed-token-correctness} complements the relative-prefix analysis
with strict absolute budgets. A response contributes at budget $B$ only if it
contains at least $B$ generated tokens; the Full point retains every valid response.
We use the same preregistered layer, geometry LR, and 10 repeated
question-disjoint 80/20 splits as in the main correctness analysis.

\begin{figure}[h]
    \centering
    \includegraphics[width=\textwidth]{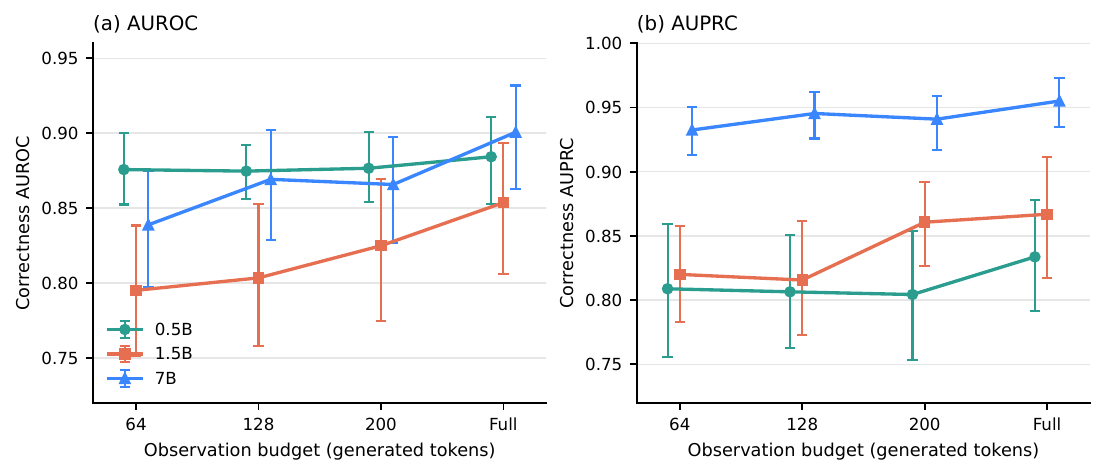}
    \caption{Correctness prediction under strict fixed-token observation.
    \textbf{(a)} AUROC. \textbf{(b)} AUPRC. Points are mean test metrics and
    error bars are 95\% split-bootstrap intervals; the positive class is a
    correct response. Retained response counts (0.5B/1.5B/7B) are
    258/258/258, 257/256/258, and 242/243/252 at 64, 128, and 200 tokens,
    respectively; Full retains 258 per model. All settings cover 86 questions
    except 0.5B at 200 tokens, which covers 84.}
    \label{fig:fixed-token-correctness}
\end{figure}

\subsection{Full Baseline Comparisons}
\label{app:baselines}
The tables below expand Table~\ref{tab:matched-baselines} to every observation
fraction. Values are mean test AUROC with 95\% confidence intervals. LLM-Check
and 7B CoE use the same five question-disjoint folds and paired stratified
question-cluster bootstrap; 0.5B CoE uses 10 repeated question-disjoint 80/20
splits and split-bootstrap intervals on its dedicated all-layer collection.
Within each table row, the largest mean is bold. The positive class is an
incorrect response (a response-level error/hallucination proxy).

\begin{table}[h]
\centering
\scriptsize
\caption{Qwen2.5-0.5B: geometry versus method-faithful LLM-Check. Geometry and
fixed scores use layer 18; ``sel.'' selects the layer using training questions
only. The 100\% row is the direct full-response comparison; earlier rows are
matched-prefix adaptations.}
\label{tab:llmcheck-05b}
\resizebox{\textwidth}{!}{%
\begin{tabular}{lcccccccc}
\toprule
Obs. & Geometry & Hidden & Hidden (sel.) & Attention & Attention (sel.) & Perplexity & Logit entropy & Window entropy \\
\midrule
20\%  & \textbf{0.82 $[0.76,0.88]$} & 0.81 $[0.74,0.87]$ & 0.79 $[0.72,0.85]$ & 0.59 $[0.49,0.68]$ & 0.62 $[0.51,0.72]$ & 0.63 $[0.56,0.70]$ & 0.62 $[0.54,0.70]$ & 0.77 $[0.71,0.83]$ \\
40\%  & 0.80 $[0.73,0.85]$ & 0.79 $[0.72,0.85]$ & 0.81 $[0.74,0.87]$ & 0.55 $[0.46,0.64]$ & 0.62 $[0.53,0.71]$ & 0.71 $[0.64,0.77]$ & 0.70 $[0.63,0.77]$ & \textbf{0.83 $[0.77,0.88]$} \\
60\%  & \textbf{0.86 $[0.81,0.91]$} & 0.80 $[0.73,0.86]$ & 0.82 $[0.76,0.88]$ & 0.44 $[0.35,0.53]$ & 0.71 $[0.64,0.79]$ & 0.74 $[0.67,0.80]$ & 0.73 $[0.65,0.80]$ & 0.83 $[0.76,0.89]$ \\
80\%  & \textbf{0.86 $[0.81,0.91]$} & 0.81 $[0.74,0.87]$ & 0.83 $[0.77,0.89]$ & 0.42 $[0.33,0.51]$ & 0.72 $[0.65,0.79]$ & 0.75 $[0.68,0.82]$ & 0.75 $[0.67,0.82]$ & 0.83 $[0.76,0.88]$ \\
100\% & \textbf{0.86 $[0.80,0.91]$} & 0.81 $[0.74,0.87]$ & 0.83 $[0.76,0.88]$ & 0.42 $[0.33,0.52]$ & 0.76 $[0.69,0.82]$ & 0.78 $[0.71,0.84]$ & 0.78 $[0.71,0.85]$ & 0.83 $[0.77,0.88]$ \\
\bottomrule
\end{tabular}%
}
\end{table}

As a controlled same-state check, at 7B and 100\% the geometry and LLM-Check
hidden scores computed from the exact same layer-21 state matrix attain AUROC
$0.91$ and $0.83$, respectively; their paired difference is $0.08$
$[0.01,0.17]$.

\begin{table}[h]
\centering
\scriptsize
\caption{Qwen2.5-7B: geometry versus method-faithful LLM-Check. Geometry and
fixed scores use layer 21; ``sel.'' selects the layer using training questions
only.}
\label{tab:llmcheck-7b}
\resizebox{\textwidth}{!}{%
\begin{tabular}{lcccccccc}
\toprule
Obs. & Geometry & Hidden & Hidden (sel.) & Attention & Attention (sel.) & Perplexity & Logit entropy & Window entropy \\
\midrule
20\%  & \textbf{0.90 $[0.85,0.94]$} & 0.89 $[0.82,0.94]$ & 0.89 $[0.82,0.94]$ & 0.77 $[0.68,0.85]$ & 0.83 $[0.75,0.90]$ & 0.55 $[0.46,0.64]$ & 0.44 $[0.34,0.54]$ & 0.73 $[0.65,0.80]$ \\
40\%  & \textbf{0.88 $[0.82,0.93]$} & 0.87 $[0.78,0.94]$ & 0.88 $[0.80,0.94]$ & 0.78 $[0.69,0.86]$ & 0.77 $[0.69,0.85]$ & 0.48 $[0.37,0.57]$ & 0.48 $[0.37,0.59]$ & 0.71 $[0.63,0.79]$ \\
60\%  & \textbf{0.91 $[0.86,0.95]$} & 0.85 $[0.77,0.92]$ & 0.86 $[0.78,0.93]$ & 0.77 $[0.68,0.86]$ & 0.75 $[0.64,0.85]$ & 0.40 $[0.30,0.50]$ & 0.58 $[0.48,0.68]$ & 0.74 $[0.67,0.81]$ \\
80\%  & \textbf{0.90 $[0.84,0.94]$} & 0.83 $[0.75,0.91]$ & 0.83 $[0.75,0.91]$ & 0.76 $[0.66,0.85]$ & 0.73 $[0.62,0.83]$ & 0.38 $[0.28,0.48]$ & 0.59 $[0.50,0.69]$ & 0.75 $[0.68,0.82]$ \\
100\% & \textbf{0.91 $[0.86,0.95]$} & 0.83 $[0.74,0.91]$ & 0.84 $[0.76,0.91]$ & 0.76 $[0.66,0.85]$ & 0.72 $[0.62,0.81]$ & 0.55 $[0.45,0.65]$ & 0.62 $[0.52,0.72]$ & 0.74 $[0.66,0.81]$ \\
\bottomrule
\end{tabular}%
}
\end{table}

~\citet{sun2026llm} report late-step AUROC up to $0.87$ using explicit
reasoning-step boundaries on different models and data. Their representation is
not directly computable from our unsegmented free-form trajectories, so we do not
present it as a matched baseline; the cross-paper contrast with our 7B 20\%
AUROC of $0.90$ is descriptive only.

\begin{table}[h]
\centering
\small
\caption{Qwen2.5-0.5B: geometry versus Chain-of-Embedding on the dedicated
all-layer collection (86 questions, five responses per question). CoE-R and
CoE-C use the released implementation and fixed published score direction.
The 100\% row is the primary method-faithful comparison.}
\label{tab:coe-05b}
\begin{tabular}{lccc}
\toprule
Obs. & Geometry & CoE-R & CoE-C \\
\midrule
20\%  & \textbf{0.84 $[0.81,0.88]$} & 0.33 $[0.27,0.40]$ & 0.60 $[0.55,0.64]$ \\
40\%  & \textbf{0.86 $[0.83,0.89]$} & 0.46 $[0.40,0.51]$ & 0.66 $[0.59,0.71]$ \\
60\%  & \textbf{0.87 $[0.84,0.90]$} & 0.46 $[0.42,0.50]$ & 0.70 $[0.63,0.75]$ \\
80\%  & \textbf{0.88 $[0.85,0.92]$} & 0.49 $[0.45,0.53]$ & 0.73 $[0.67,0.79]$ \\
100\% & \textbf{0.87 $[0.84,0.91]$} & 0.51 $[0.46,0.56]$ & 0.75 $[0.69,0.80]$ \\
\bottomrule
\end{tabular}
\end{table}

\begin{table}[!h]
\centering
\small
\caption{Qwen2.5-7B: geometry versus Chain-of-Embedding. CoE-R and CoE-C use the
published score direction and the embedding output plus every transformer layer.
The 100\% row is the primary method-faithful comparison.}
\label{tab:coe-7b}
\begin{tabular}{lccc}
\toprule
Obs. & Geometry & CoE-R & CoE-C \\
\midrule
20\%  & \textbf{0.90 $[0.84,0.95]$} & 0.53 $[0.43,0.62]$ & 0.56 $[0.46,0.66]$ \\
40\%  & \textbf{0.88 $[0.82,0.93]$} & 0.66 $[0.57,0.75]$ & 0.68 $[0.59,0.76]$ \\
60\%  & \textbf{0.91 $[0.86,0.95]$} & 0.71 $[0.62,0.79]$ & 0.73 $[0.64,0.81]$ \\
80\%  & \textbf{0.90 $[0.85,0.94]$} & 0.72 $[0.63,0.80]$ & 0.74 $[0.66,0.82]$ \\
100\% & \textbf{0.91 $[0.86,0.95]$} & 0.75 $[0.65,0.83]$ & 0.76 $[0.67,0.83]$ \\
\bottomrule
\end{tabular}
\end{table}

\paragraph{Prompts.}
We use the \textbf{medium} chain-of-thought prompt for every experiment reported
in the main paper. The short and long templates below are retained for
completeness but are not pooled into the reported results.

\begin{tcolorbox}[title={\textbf{Short CoT Prompt}}]
\textit{Go with your instinct. Write only the essential steps --
no extra explanation. Answer in \texttt{\textbackslash boxed\{\}}.}
\end{tcolorbox}

\begin{tcolorbox}[title={\textbf{Medium CoT Prompt}}]
\textit{Think step by step. Show your reasoning, then give
the final answer in \texttt{\textbackslash boxed\{\}}.}
\end{tcolorbox}

\begin{tcolorbox}[title={\textbf{Long CoT Prompt}}]
\textit{Overthink this. Consider multiple approaches.
Solve step by step, then second-guess yourself.
Check your work using a different method.
Ask: what could I be missing? What if I made an error?
Keep thinking until you're fully satisfied.
Answer in \texttt{\textbackslash boxed\{\}}.}
\end{tcolorbox}
============================================================================
\end{document}